\documentclass[letterpaper]{article} 
\usepackage[preprint]{aaai2027}  
\usepackage[hyphens]{url}  
\usepackage{graphicx} 
\usepackage{natbib}  
\usepackage{caption} 
\usepackage{algorithm}
\usepackage{algorithmic}

\usepackage{newfloat}
\usepackage{listings}
\DeclareCaptionStyle{ruled}{labelfont=normalfont,labelsep=colon,strut=off} 
\floatstyle{ruled}
\newfloat{listing}{tb}{lst}{}
\floatname{listing}{Listing}

\usepackage{booktabs}

\usepackage{amsmath}
\usepackage{amssymb}
\usepackage{amsthm}
 
\theoremstyle{plain}
\newtheorem{theorem}{Theorem}
\newtheorem{lemma}{Lemma}
\newtheorem{proposition}{Proposition}
\newtheorem{corollary}{Corollary}
\theoremstyle{definition}
\newtheorem{definition}{Definition}
\newtheorem{assumption}{Assumption}
\theoremstyle{remark}
\newtheorem{remark}{Remark}
 
\newcommand{\E}{\mathbb{E}}

\newcommand{\Prob}{\Pr}
\newcommand{\pllm}{p_{\text{LLM}}}
\newcommand{\me}{m_e}
\newcommand{\mo}{m_o}
\newcommand{\envp}{\varepsilon^{+}}
\newcommand{\envm}{\varepsilon^{-}}

\title{Bilevel Coordinated Reflection: A Game-Theoretic Approach to Multi-Agent LLM Systems}
\author{
    Yihang Chen\textsuperscript{\rm 1}\thanks{Equal contribution.},
    Yuxiang Chen\textsuperscript{\rm 1}\footnotemark[1],
    Yuxuan Huang\textsuperscript{\rm 2},
    Meng Fang\textsuperscript{\rm 2},
    Weilin Luo\textsuperscript{\rm 3},
    Jun Wang\textsuperscript{\rm 1}\thanks{Corresponding author.}
}
\affiliations{
    \textsuperscript{\rm 1}UCL Centre for Artificial Intelligence \quad
    \textsuperscript{\rm 2}University of Liverpool  \quad
    \textsuperscript{\rm 3}Huawei\\
}

\begin{document}
\maketitle

\begin{abstract}
Multi-agent LLM systems commonly use an orchestrator to decompose a task
for a team of workers and then improve through textual reflection. Despite
strong empirical results, these systems lack a unified account of
coordination, memory improvement, and the role of external verification.
We model orchestrator--worker interaction as a bilevel coordination game:
under bounded coupling, the workers' local-update game is an approximate
potential game whose equilibrium slack is controlled by decomposition
quality. We then analyse reflection as stochastic movement over semantic
memory states. For free-form reflection, we derive a finite-time upper
bound, prove worst-case tightness, and give a positive lower bound under a
falsifiable persistent-harm condition. We further prove an
information-theoretic impossibility result: no gate that observes only the
generated transcript can improve uniformly over text-indistinguishable
environments, whereas an environment-grounded gate can. Motivated by this
separation, we introduce Stochastic Reflective Memory Ascent (SRMA), which
accepts a candidate memory only after a grounded evaluation risk strictly
decreases. Under calibration and non-degenerate corrective mass, SRMA
converges exactly, geometrically or polynomially; matching constructions
show that both rate regimes are order-tight. We also provide confidence
gating for stochastic evaluation and re-anchoring guarantees for
piecewise-stationary environments. Experiments instantiate these objects
with environment-grounded metrics and test the predicted coordination and
drift laws. On 500 SWE-bench instances, the complete Kimi-based system
resolves $72.2\%$ versus a $70.8\%$ public mini-SWE-agent reference. Code available at \textcolor{blue}{\url{https://github.com/YihangChen9/Bilevel-Coordinated-Reflection}}.
\end{abstract}

\begin{figure*}[t]
    \centering
    \includegraphics[width=0.98\linewidth]{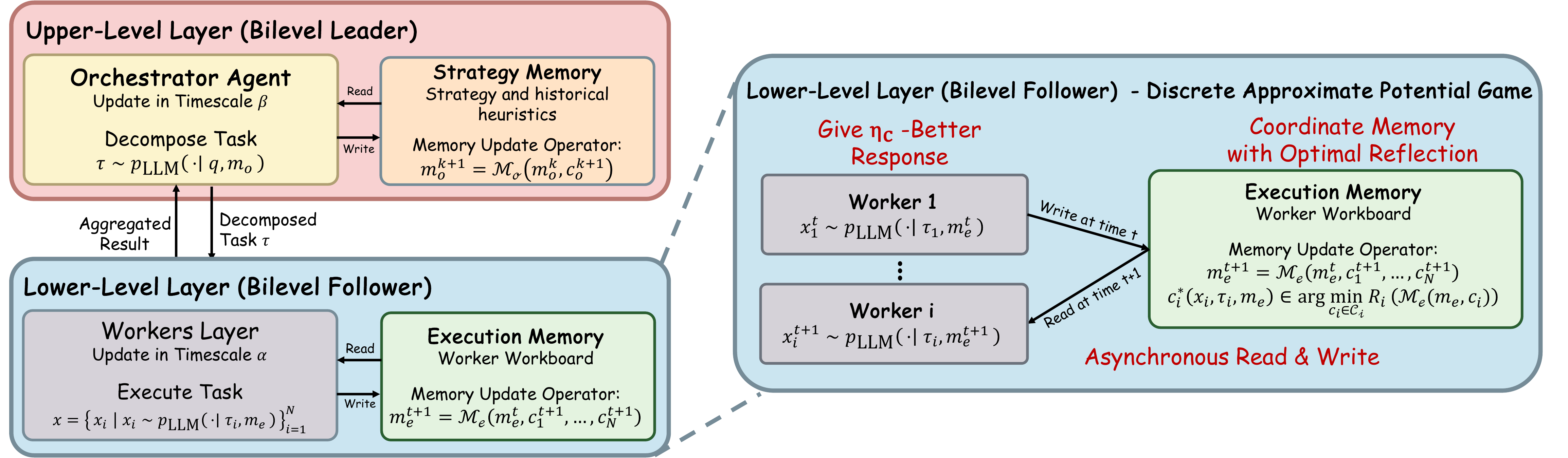}
    \caption{Bilevel coordinated reflection. The orchestrator (leader)
    selects a decomposition $\tau$ and updates strategy memory $m_o$ on the
    slower timescale; workers (followers) update execution memory $m_e$ via
    $\eta_c$-better responses on the faster timescale. Under bounded
    coupling, the followers' subgame is an approximate potential game with
    slack $\eta_c \le 2 d_{\max}\kappa$, while verifier-gated SRMA
    separately governs which memory proposals are committed.}
        \label{fig:game}
\end{figure*}
\section{Introduction}
\label{sec:intro}
Multi-agent LLM systems have become a common recipe for tasks too large or
structured for a single agent: an orchestrator decomposes the task, worker
models solve the pieces, and the team improves by \emph{reflecting}---writing
critiques, hypotheses, and lessons into a shared textual memory that
conditions subsequent generations
\citep{wu2023autogen,hong2024metagpt,shinn2023reflexion,benkovich2026agyn,qian2025scaling}.
Because model weights are frozen at test time, memory editing is the
principal adaptation channel~\citep{wang2025memento2,amem2025,zhang2025memory_survey},
and such loops often work better when grounded by a test harness,
simulator, execution engine, or formal checker.

The dominant account of these systems is nevertheless procedural. Existing
frameworks~\citep{zhang2025aflow,hu2025adas,dang2025evolving,wang2025evoagentx}
specify who communicates with whom and which buffer is updated,
but not the strategic object that the agents stabilise to or the quantity
that reflection improves. This leaves three unresolved questions. First,
how does the orchestrator's decomposition quality control worker
coordination? Second, when does unconditional reflection plateau rather
than converge? Third, why can an external verifier succeed where a stronger
text-only critic may still fail?

We address these questions in a single framework. The
orchestrator--worker pipeline is modelled as a bilevel coordination game
whose follower subgame is an approximate potential game, and textual
memory editing as a stochastic process over a discrete semantic state
space. For free-form reflection, a one-sided drift condition yields a
finite-time upper bound that is tight in the worst case; a universal
positive floor requires an additional, explicitly testable persistent-harm
condition---unconditional commitment alone is not enough.

We then isolate the informational role of verification: in two
environments with identical text-generation laws but opposite meanings for
the same reflections, any possibly randomised, history-dependent gate that
observes only the transcript behaves identically and therefore cannot
improve both---even an ideal text-only judge---whereas a grounded verifier
distinguishes the pair and recovers geometric convergence.

Motivated by this separation, we introduce Stochastic Reflective Memory
Ascent (SRMA), which commits a candidate memory only when a fixed grounded
evaluation protocol certifies a strict decrease in verifier risk. Under
calibration and non-degenerate corrective mass, SRMA converges exactly at
order-tight geometric or polynomial rates; a confidence gate handles
stochastic probes, and re-anchoring restores per-segment convergence under
piecewise stationarity.

The theory is instantiated on a hidden-cap resource contest, Overcooked
with an exact BFS value table, and SWE-bench~\citep{jimenez2024swebench}.
The controlled environments expose the strategic, memory, and drift
quantities directly without an LLM-as-judge; on SWE-bench the complete
Kimi-based system resolves $361/500$ instances ($72.2\%$) versus
$70.8\%$ for the public mini-SWE-agent v2 reference.

In summary, we contribute: \textbf{(1)} a bilevel coordination game
linking decomposition coupling to follower equilibrium slack
(Sec.~\ref{sec:formulation}); \textbf{(2)} a two-sided drift analysis of
free-form reflection, tight in the worst case, with a universal lower
bound under persistent harmful commitment (Sec.~\ref{sec:memory_driven});
\textbf{(3)} an impossibility theorem for self-contained text-only gates,
with a grounded comparator that converges geometrically
(Sec.~\ref{sec:impossibility}); \textbf{(4)} SRMA, with exact convergence,
order-tight rates, and a finite-probe confidence extension
(Sec.~\ref{sec:srma}); and \textbf{(5)} mechanism-level validation on
Resource Contest and Overcooked plus end-to-end results on SWE-bench
(Sec.~\ref{sec:experiments}).

\section{Related Work}
\label{sec:related}

\textbf{Multi-agent LLM frameworks.} Orchestrator--worker architectures
such as AutoGen~\citep{wu2023autogen}, MetaGPT~\citep{hong2024metagpt} and
Agyn~\citep{benkovich2026agyn} show strong empirical performance
but offer no convergence analysis; failure modes such as hallucination
cascades are documented empirically~\citep{agenthallu2026,cemri2025whyfail}.
We provide the missing game-theoretic and stochastic-approximation
foundations.

\textbf{Self-reflection, self-evaluation, and grounding.}
Reflexion~\citep{shinn2023reflexion} and
Self-Refine~\citep{madaan2023selfrefine} improve outputs by appending
self-generated critiques but may plateau, and correlated self-evaluation
bias \citep{zheng2023llmjudge,panickssery2024llmevaluators,council2026}
weakens model-based judges in practice. Our drift analysis separates a
worst-case floor from the persistent-harm condition needed for a universal
lower bound, and our indistinguishable-environment theorem shows that
without an environment-dependent signal even an ideal text-only gate
cannot be uniformly correct.

\textbf{Potential games and drift analysis.} Our followers' subgame builds
on exact and approximate potential games
\citep{monderer1996potential,candogan2011flows,christodoulou2014approximate}
and weakly coupled team problems~\citep{srikant1992nash}. The convergence
analysis uses Foster--Lyapunov drift~\citep{hajek1982hitting,meyn2009markov},
classical stochastic approximation
\citep{robbins1951stochastic,borkar2008stochastic,bertsekas2000gradient} and,
for the gated regime, multiplicative and variable drift theorems from
randomised search heuristics
\citep{doerr2012multiplicative,johannsen2010thesis,lehre2021variable};
the recursion $e_{t+1}\le e_t - c\,e_t^{1+\beta}$ is the discrete
stochastic analogue of Polyak--{\L}ojasiewicz-type
conditions~\citep{karimi2016linear,chung1954stochastic}. Two-timescale
bilevel structure follows~\citet{borkar1997two,hong2023two}.

\section{Methodology}
\label{sec:method}

Longer derivations are deferred to the supplementary material.

\subsection{Problem Formulation: Bilevel Coordination Game}
\label{sec:formulation}

We formalise the resolution of a complex user query $q \in \mathcal{Q}$.
The objective is a joint structured output $x \in \mathcal{X}$ maximising
a global utility $U(x)$ (logical correctness, constraint satisfaction). In
a naive single-agent paradigm the entire output is generated directly from
the query via the frozen LLM kernel, $x \sim \pllm(\cdot \mid q)$, which
for large tasks induces context dilution and reasoning
degradation~\citep{liu2024lost,levy2024same,du2025contextlength}. Contemporary systems instead
let an orchestrator partition the task among
workers~\citep{wu2023autogen,hong2024metagpt,liu2025selectdecompose}.

We model this as a \textit{bilevel coordination game}. The orchestrator
(Leader) generates a strategy profile
$\tau = (\tau_1, \ldots, \tau_N) \in \mathcal{T}$,
\begin{equation}
\label{eq:assign}
    \tau \sim \pllm(\cdot \mid q),
\end{equation}
assigning subtask $\tau_i$ to worker $i$ (Follower), who generates a local
sub-solution $x_i \sim \pllm(\cdot \mid \tau_i)$; the global output is
$x = (x_1, \ldots, x_N)$.

Unlike the idealised independent decomposition of classical
potential-game analyses~\citep{monderer1996potential}, real multi-agent
LLM systems exhibit non-trivial cross-worker interactions: shared
variables, common interfaces, joint
constraints~\citep{agenthallu2026}. We adopt a \emph{weakly
coupled} decomposition in the spirit
of~\citet{srikant1992nash,candogan2011flows}.

\begin{assumption}[Weakly Coupled Decomposability]
\label{assum:decomposability}
Each worker action set $\mathcal X_i$ is finite. The worker payoff is the local objective $g_i(x;\tau):=u_i(x_i\mid\tau_i)$, while the system-level objective is $U$. The global utility admits
\begin{align}
\label{eq:weakly_coupled}
    U(x) ={}& \sum_{i=1}^{N} u_i(x_i \mid \tau_i) \notag\\
    &+ \sum_{(i,j) \in \mathcal{E}} \psi_{ij}(x_i, x_j \mid \tau_i, \tau_j),
\end{align}
where $\mathcal{E}$ is an undirected interaction graph induced by $\tau$, with each edge counted once, and
$\kappa := \sup_{(i,j) \in \mathcal{E}} \sup_{x_i, x_j}
|\psi_{ij}(x_i, x_j \mid \tau_i, \tau_j)| < \infty$.
Let $\mathcal{N}_i$ be worker $i$'s coupled neighbours and
$d_{\max} := \max_i |\mathcal{N}_i|$.
\end{assumption}

When $\kappa = 0$ the system reduces to the independent case; $\kappa$ and
$d_{\max}$ jointly quantify decomposition quality. Since LLM generation is
stochastic, the system objective is the \emph{expected} global utility
$\E[U(x)]$.

\begin{lemma}[Approximate Potential Game]
\label{lemma:potential_game}
Under Assumption~\ref{assum:decomposability} and fixed $\tau$, the
workers' subgame is an $\eta_c$-approximate potential game with potential
$\E[U(x)]$ and slack
\begin{equation}
    \eta_c \;\le\; 2\, d_{\max}\, \kappa .
\end{equation}
\end{lemma}
\begin{proof}
If worker $i$ unilaterally deviates from $x_i^t$ to $x_i^{t+1}$,
\begin{align}
    &\E[U(x_i^{t+1}, x_{-i}^t)] - \E[U(x_i^t, x_{-i}^t)] \notag \\
    &\quad= \E[u_i(x_i^{t+1} \mid \tau_i)] - \E[u_i(x_i^t \mid \tau_i)]
      + \Delta_i^\psi,
\end{align}
where the coupling residual $\Delta_i^\psi$ sums at most $d_{\max}$ terms
each bounded by $2\kappa$ (since $|\psi_{ij}| \le \kappa$), so
$|\Delta_i^\psi| \le 2 d_{\max}\kappa =: \eta_c$. Every unilateral
deviation thus changes the potential within $\eta_c$ of the local utility
change~\citep{candogan2011flows,christodoulou2014approximate}.
\end{proof}

A rational worker performs $\eta_c$-better-response updates:
$\E[u_i(x_i^{t+1} \mid \tau_i)] - \E[u_i(x_i^t \mid \tau_i)] > \eta_c$. If
no worker has such a deviation, the current profile is by definition
already an $\eta_c$-approximate Nash equilibrium, so the dynamics below
are well defined in all cases.

\begin{theorem}[Convergence of the Followers' Subgame]
\label{thm:nash_ideal}
Under Lemma~\ref{lemma:potential_game}, iterated
$\eta_c$-better-response updates converge in finitely many steps to a
profile $x^*(\tau)$ satisfying, for every worker $i$,
\begin{equation}
\E[g_i(x_i^*,x_{-i}^*;\tau)]\ge
\max_{x_i'\in\mathcal X_i}
\E[g_i(x_i',x_{-i}^*;\tau)]-\eta_c.
\end{equation}
Thus $x^*(\tau)$ is an $\eta_c$-approximate pure-strategy Nash equilibrium
of the explicitly defined local-payoff game.
\end{theorem}
\begin{proof}[Proof sketch]
Each update raises the potential $\E[U(x^t)]$ by a strictly positive
amount (Lemma~\ref{lemma:potential_game}); $\mathcal{X}$ finite and $U$
bounded imply finite termination. Full proof in the supplementary material.
\end{proof}

\paragraph{Leader's objective and decomposition quality.}
The orchestrator anticipates the followers' equilibrium and solves
$\tau^*(q) \in \arg\max_{\tau} \E[U(x^*(\tau))]$. Because
$\eta_c = 2 d_{\max}(\tau)\kappa(\tau)$ depends on $\tau$, the leader's
objective contains an explicit decomposition-quality term:

\begin{corollary}[Leader's Decomposition Trade-off]
\label{cor:leader}
Let $J_{\mathrm{loc}}(\tau) := \sum_i \max_{x_i} \E[u_i(x_i \mid \tau_i)]$
and $C(\tau) := d_{\max}(\tau)\,\kappa(\tau)$. For any
$\eta_c$-approximate equilibrium $x^*(\tau)$,
\begin{equation}
\label{eq:leader_tradeoff}
    \E[U(x^*(\tau))] \;\ge\; J_{\mathrm{loc}}(\tau)
    \;-\; \tfrac{5}{2}\, N\, C(\tau).
\end{equation}
Hence the leader maximises a lower bound that trades achievable local
utility against coupling: a good decomposition simultaneously raises
$J_{\mathrm{loc}}$ and shrinks $C(\tau)$. (Proof in the supplementary material.)
\end{corollary}

\subsection{Dual-Memory Drift Dynamics and Hallucination Floors}
\label{sec:memory_driven}

LLM weights are frozen, so adaptation proceeds by editing external,
non-parametric memories: an execution memory $\me\in\mathcal M_e$ shared by
workers and a strategy memory $\mo\in\mathcal M_o$ used by the orchestrator.
For a fixed decomposition $\tau$, let
\[
J_i(\me\mid\tau_i)=
\E_{x_i\sim\pllm(\cdot\mid\tau_i,\me)}[u_i(x_i\mid\tau_i)]
\]
and rescale utility so that the sub-optimality
$V_t:=V_i(\me^t)=J_i^*-J_i(\me^t\mid\tau_i)$ lies in $[0,1]$.
Let $\mathcal F_t$ denote the history up to the $t$-th memory update.

The key distinction is whether a proposed reflection is committed
unconditionally or evaluated before it enters memory. Unconditional
commitment alone does \emph{not} imply a positive asymptotic error: a
universal lower bound requires an explicit condition that harmful
commitments keep injecting non-vanishing expected error. We therefore
separate an upper guarantee, its worst-case tightness, and a genuine
lower bound under persistent harmful drift.

\subsubsection{Regime A: free-form reflection.}
When every generated reflection is appended, corrective information and
hallucinated information~\citep{huang2025hallucination,ji2024hallucination}
are mixed in the same update. We summarise their
net conditional effect by the following one-sided drift condition.

\begin{assumption}[One-Sided Free-Form Drift]
\label{assum:lyapunov_drift}
There exist $\gamma_t\in(0,1]$ and $\nu_t\ge 0$ such that
\begin{equation}
\label{eq:drift_t}
    \E[V_{t+1}\mid\mathcal F_t]
    \le (1-\gamma_t)V_t+\nu_t.
\end{equation}
Here $\gamma_t V_t$ is the available corrective drift and $\nu_t$ is the
mean residual error load from committed, ungrounded content; $\nu_t$ is a
first-moment quantity, not a variance.
\end{assumption}

\begin{theorem}[Finite-Time Upper Bound]
\label{thm:regimeA-upper}
If $\gamma_t\ge\underline\gamma>0$ and
$\nu_t\le\overline\nu$, then, for $e_t:=\E[V_t]$,
\begin{equation}
\label{eq:free_upper}
 e_T\le (1-\underline\gamma)^T e_0
 +\frac{\overline\nu}{\underline\gamma}
  \bigl(1-(1-\underline\gamma)^T\bigr).
\end{equation}
Consequently,
$\limsup_{T\to\infty}e_T\le
\min\{1,\overline\nu/\underline\gamma\}$. (Proof in the supplementary material.)
\end{theorem}

Theorem~\ref{thm:regimeA-upper} is an upper guarantee only; the next
result is the strongest conclusion available from
Assumption~\ref{assum:lyapunov_drift} alone.

\begin{proposition}[Worst-Case Tightness]
\label{prop:regimeA-tight}
For every $\gamma\in(0,1]$ and $\nu\in(0,\gamma]$, there exists a free-form
process satisfying Assumption~\ref{assum:lyapunov_drift} with
$\gamma_t\equiv\gamma$ and $\nu_t\equiv\nu$ such that
\begin{equation}
\label{eq:minimax_tight}
    \lim_{T\to\infty}\E[V_T]=\frac{\nu}{\gamma}.
\end{equation}
Hence the upper bound $\nu/\gamma$ cannot be uniformly improved over the
one-sided drift class.
\end{proposition}
\begin{proof}[Proof sketch]
The deterministic recursion $V_{t+1}=(1-\gamma)V_t+\nu$ with
$0<\nu\le\gamma$ maps $[0,1]$ into itself, attains \eqref{eq:drift_t} with
equality, and converges to its unique fixed point $\nu/\gamma$.
\end{proof}

A lower bound that applies to \emph{every} process requires a lower drift
condition, directly testable by regressing the next-step error on the
current error in free-form trajectories.

\begin{assumption}[Persistent Harmful Commitment]
\label{assum:persistent_noise}
There exist $\overline\gamma\in(0,1]$ and
$\underline\nu\in(0,\overline\gamma]$ such that, on every reachable state,
\begin{equation}
\label{eq:lower_drift}
    \E[V_{t+1}\mid\mathcal F_t]
    \ge (1-\overline\gamma)V_t+\underline\nu.
\end{equation}
The parameter $\overline\gamma$ upper-bounds how much of the current error
can be removed in one expected update, whereas $\underline\nu>0$ is a
persistent net error load that remains because harmful reflections are
committed without screening.
\end{assumption}

\begin{theorem}[Universal Lower Bound]
\label{thm:regimeA-lower}
Under Assumption~\ref{assum:persistent_noise},
\begin{equation}
\label{eq:free_lower}
 e_T\ge (1-\overline\gamma)^T e_0
 +\frac{\underline\nu}{\overline\gamma}
  \bigl(1-(1-\overline\gamma)^T\bigr),
\end{equation}
and therefore
\begin{equation}
\label{eq:positive_floor}
    \liminf_{T\to\infty}e_T
    \ge \frac{\underline\nu}{\overline\gamma}>0.
\end{equation}
(Proof in the supplementary material.)
\end{theorem}

\begin{corollary}[Two-Sided Error Tube]
\label{cor:free_tube}
If Assumptions~\ref{assum:lyapunov_drift} and
\ref{assum:persistent_noise} both hold, then
\begin{equation}
\frac{\underline\nu}{\overline\gamma}
\le \liminf_{T\to\infty}e_T
\le \limsup_{T\to\infty}e_T
\le \frac{\overline\nu}{\underline\gamma}.
\end{equation}
When the two conditional drift bounds match,
$\underline\gamma=\overline\gamma=\gamma$ and
$\underline\nu=\overline\nu=\nu$, the mean error converges exactly to
$\nu/\gamma$.
\end{corollary}

An operational corollary in the supplementary material re-expresses the
tube via estimable per-step correction and harm rates.

\paragraph{Leader's outer loop.}
On the slower timescale, define
$V_o(\mo^k)=\Phi^*-\Phi(\mo^k)$ for
$\Phi(\mo^k)=\E[U(x)\mid\me^\infty(\tau(\mo^k))]$.
Under the analogous upper drift condition with
$\gamma_o^k\ge\gamma_{o,\min}>0$ and $\nu_{o,k}\le\nu_{o,\max}$, the
same affine recursion yields the finite-episode bound
$\E[V_o(\mo^K)]\le(1-\gamma_{o,\min})^K V_o(\mo^0)
+\nu_{o,\max}/\gamma_{o,\min}$. We use only this finite-episode
statement and make no asymptotic leader-regret claim.

\subsection{Why Grounding Is Necessary: Impossibility of Self-Contained Gates}
\label{sec:impossibility}

The fundamental informational requirement is \emph{grounding}: access to a
signal whose law depends on the environment rather than only on the
generated transcript. We formalise this through a pair of environments
that are indistinguishable at the text level.

\paragraph{Text processes and gates.}
Let a memory be a finite reflection sequence
$m=(c_1,\ldots,c_t)\in\mathcal C^*$ with append operation
$m\oplus c=(c_1,\ldots,c_t,c)$. Fix an initial memory $m^0$ and a proposal
kernel $P(\cdot\mid m)$ over $\mathcal C$. An environment $\varepsilon$
assigns a sub-optimality $V^\varepsilon(m)\in[0,1]$ to every reachable
memory. Across the class considered below, the environment changes this
semantic value but not the proposal kernel or any other text-level law.

\begin{definition}[Self-Contained Gate]
\label{def:selfgate}
A \emph{self-contained gate} is any possibly randomised,
history-dependent acceptance rule measurable with respect to the generated
text process and its internal randomness only. A \emph{grounded gate} may
additionally observe an environment-dependent signal, such as realised
reward, simulator state, test execution, or a formal-checker result.
\end{definition}

This class contains text-only LLM-as-judge systems. Correlated-evaluation
bias can make such judges weaker in practice
\citep{panickssery2024llmevaluators}; the result below applies even to an
ideal gate with unlimited text-processing capacity.

\paragraph{Ambiguous-pair construction.}
Let $C_0,C_1\subset\mathcal C$ be disjoint and satisfy
$P(C_0\mid m)=P(C_1\mid m)=\mu\in(0,\tfrac12]$ for every reachable $m$;
all remaining proposals are inert. Fix $\kappa\in(0,1)$ and define
$f_0(v)=(1-\kappa)v$ and
$f_1(v)=\kappa+(1-\kappa)v$. In environment $\envp$, an accepted proposal
from $C_a$ applies $f_a$ to the current error; in environment $\envm$, the
roles of $C_0$ and $C_1$ are swapped. Thus the same text is corrective in
one environment and harmful in the other. Let
$e_T^\varepsilon=\E[V^\varepsilon(m^T)]$ and assume both environments start
at the same $e_0\le\tfrac12$.

\begin{theorem}[Self-Gating Impossibility]
\label{thm:selfgate}
For every self-contained gate and every horizon $T$,
\begin{equation}
\label{eq:selfgate_floor}
 \max\{e_T^{\envp},e_T^{\envm}\}\ge e_0.
\end{equation}
Moreover, if $e_0<\tfrac12$ and the gate accepts at least one proposal from
$C_0\cup C_1$ with positive probability by time $T$, then the inequality is
strict. In contrast, the free-form rule accepts everything and satisfies
$e_T^\varepsilon\to\tfrac12$ in both environments, whereas the grounded
gate that observes $V^\varepsilon$ accepts only the corrective class and
satisfies $e_T^\varepsilon=e_0(1-\kappa\mu)^T\to0$ in both environments.
\end{theorem}
\begin{proof}[Proof sketch]
Couple both environments with shared proposal and gate randomness; the
accepted class-label sequence is then identical under $\envp$ and
$\envm$, and the reflection identity $f_{1-a}(v)=1-f_a(1-v)$ yields
$e_T^{\envp}+e_T^{\envm}\ge2e_0$, strictly when $e_0<\tfrac12$ and an
ambiguous proposal is accepted with positive probability. The free-form
and grounded rates follow from the induced affine recursions. Full proof
in the supplementary material.
\end{proof}

\begin{remark}[Scope of the impossibility result]
\label{rem:selfgate_scope}
The theorem is minimax over text-indistinguishable environments; textual
self-evaluation remains useful when the transcript itself certifies
correctness (a fully checkable proof). When truth depends on external
state---hidden caps, API responses, simulator state, an evolving
repository---judge capacity cannot substitute for grounding.
\end{remark}

\subsection{SRMA: Verifier-Gated Reflection}
\label{sec:srma}

Theorem~\ref{thm:selfgate} establishes why the gate must have access to an
environment-separating signal. Exact convergence additionally requires the
gate to compare a fixed error functional of the memory state, rather than
two uncontrolled one-shot samples from a stochastic generator. We therefore
separate the stochastic \emph{proposal} mechanism from the grounded
\emph{evaluation} protocol.

\begin{definition}[Verifier and Evaluation Risk]
\label{def:verifier}
A verifier is a deterministic map
$\mathcal V:\mathcal X\times\mathcal T\to\mathcal S$ with deterministic
score $\rho:\mathcal S\to[0,1]$. Let
$g_i:\mathcal T_i\times\mathcal M_e\to\mathcal X_i$ be a fixed
deterministic evaluation protocol, such as an exact planner or decoding
with fixed randomness. The verifier risk of memory $\me$ is
\begin{equation}
\label{eq:verifier_risk}
 R_i(\me):=\rho\!\left(\mathcal V(g_i(\tau_i,\me),\tau_i)\right).
\end{equation}
The pair $(\mathcal V,g_i)$ is fixed independently of the reflection
proposal distribution. It is \emph{grounded} when its score depends on an
environment signal that is not determined by the generated transcript
alone. Grounding supplies information; calibration below connects the score
to task utility.
\end{definition}

\begin{definition}[Verifier-Gated SRMA Update]
\label{def:gated_update}
Given $\me^t$, compute the evaluation output
$x_i^t=g_i(\tau_i,\me^t)$ and diagnostic
$s_i^t=\mathcal V(x_i^t,\tau_i)$. Sample a reflection
$c_i^{t+1}\sim\pllm(\cdot\mid x_i^t,s_i^t,\tau_i,\me^t)$ and form
$\widetilde\me^{t+1}=\mathcal M_e(\me^t,c_i^{t+1})$. Accept iff
\begin{equation}
\label{eq:gate}
    R_i(\widetilde\me^{t+1})<R_i(\me^t).
\end{equation}
On acceptance set $\me^{t+1}=\widetilde\me^{t+1}$; otherwise retain
$\me^{t+1}=\me^t$.
\end{definition}

Gating on two stochastic one-shot outputs would not suffice: sample
variation could accept a memory with worse expected performance. Exact
guarantees therefore assume deterministic or exact expected-risk
evaluation; a finite-sample extension follows below.

\begin{assumption}[Verifier Calibration]
\label{assum:admissible}
There exists $L<\infty$ such that, for every reachable memory,
\begin{equation}
\label{eq:calibration}
    V_i(\me)\le L R_i(\me).
\end{equation}
Thus zero verifier risk certifies zero task sub-optimality. This assumption
is appropriate for exact value tables and complete formal checkers; on
incomplete test suites, our theorem concerns verifier risk only.
\end{assumption}

\begin{assumption}[Non-Degenerate Corrective Mass]
\label{assum:acceptance}
There exist $c_1\in(0,1]$ and $\beta\in[0,1]$ such that, whenever
$R_t:=R_i(\me^t)>0$,
\begin{equation}
 p_t:=\Prob[\mathrm{accept}\mid\mathcal F_t]\ge c_1 R_t^\beta.
\end{equation}
\end{assumption}

\begin{assumption}[Proportional Accepted Decrement]
\label{assum:decrement}
There exists $c_2\in(0,1]$ such that
\begin{equation}
\E[R_t-R_{t+1}\mid\mathcal F_t,\mathrm{accept}]
\ge c_2 R_t.
\end{equation}
\end{assumption}

\begin{proposition}[Monotone Multiplicative Drift]
\label{prop:drift_decomposition}
Under Definition~\ref{def:gated_update}, $R_{t+1}\le R_t$ almost surely and
\begin{gather}
\label{eq:drift_decomp}
 \E[R_{t+1}\mid\mathcal F_t]=R_t-p_t\Delta_t,\\
 \Delta_t:=\E[R_t-R_{t+1}\mid\mathcal F_t,\mathrm{accept}].\notag
\end{gather}
Under Assumptions~\ref{assum:acceptance}--\ref{assum:decrement}, with
$c=c_1c_2$,
\begin{equation}
\label{eq:mult_drift}
 \E[R_{t+1}\mid\mathcal F_t]
 \le R_t-cR_t^{1+\beta}.
\end{equation}
\end{proposition}

\begin{theorem}[Exact Verifier Convergence and Rates]
\label{thm:gated_rates}
Under Assumptions~\ref{assum:acceptance}--\ref{assum:decrement}, let
$r_t:=\E[R_t]$ and $c=c_1c_2$. Then $R_t\to0$ almost surely and
\begin{align}
\beta=0:\quad &r_T\le(1-c)^T r_0,
&&\text{(geometric)},\label{eq:rate_geo}\\
\beta\in(0,1]:\quad &r_T\le
\bigl(r_0^{-\beta}+c\beta T\bigr)^{-1/\beta},
&&\text{(polynomial)}.\label{eq:rate_poly}
\end{align}
If Assumption~\ref{assum:admissible} also holds, then
$\E[V_i(\me^T)]\le Lr_T$ and hence the task sub-optimality converges to zero
at the same rate up to the factor $L$.
\end{theorem}
\begin{proof}[Proof sketch]
Taking expectations in \eqref{eq:mult_drift} and applying Jensen's
inequality to $z\mapsto z^{1+\beta}$ gives
$r_{t+1}\le r_t-cr_t^{1+\beta}$; the rates follow by the standard
multiplicative/variable-drift comparison. $R_t$ is non-increasing and
non-negative, hence converges almost surely, and $r_t\to0$ forces the
limit to be zero. Calibration transfers the bound to the utility gap.
\end{proof}

\begin{proposition}[Rate Tightness for Verifier-Gated Reflection]
\label{prop:rate_tight}
For every $c_1\in(0,1]$, $c_2\in(0,\tfrac12]$,
$\beta\in(0,1]$, and $r_0\in(0,1]$, there exists a process satisfying
Assumptions~\ref{assum:acceptance}--\ref{assum:decrement} with equality
such that, for $c=c_1c_2$,
\begin{equation}
\label{eq:rate_lower}
 \E[R_T]\ge
 \bigl(r_0^{-\beta}+4c\beta T\bigr)^{-1/\beta}
 \qquad\text{for every }T.
\end{equation}
For $\beta=0$, the analogous construction gives
$\E[R_T]=(1-c)^Tr_0$ exactly. Hence the geometric rate is exact and the
polynomial exponent $T^{-1/\beta}$ is tight up to a constant factor in the
time scale.
\end{proposition}
\begin{proof}[Proof sketch]
Accept with probability $c_1R_t^\beta$ and set $R_{t+1}=(1-c_2)R_t$ on
acceptance, so both assumptions hold with equality. For $\beta>0$,
$Y_t=R_t^{-\beta}$ has constant expected increment, and convexity of
$y\mapsto y^{-1/\beta}$ gives \eqref{eq:rate_lower}; for $\beta=0$,
$\E[R_{t+1}\mid\mathcal F_t]=(1-c_1c_2)R_t$. Full proof in the
supplementary material.
\end{proof}

\begin{proposition}[Confidence-Gated Stochastic Evaluation]
\label{prop:finite_probe}
Suppose deterministic $R_i(\me)$ is unavailable and instead
$R_i(\me)=\E[Z(\me)]$ for an i.i.d. score $Z(\me)\in[0,1]$. At round $t$,
estimate the current and candidate risks with $K_t$ independent probes and
let
\begin{equation}
 a_t=\sqrt{\frac{\log(4/\delta_t)}{2K_t}}.
\end{equation}
Accept only when
$\widehat R_t(\widetilde\me^{t+1})+a_t
<\widehat R_t(\me^t)-a_t$.
Then, with probability at least $1-\sum_t\delta_t$, every accepted update
strictly decreases the true expected verifier risk.
\end{proposition}
\begin{proof}[Proof sketch]
Hoeffding's inequality bounds each of the two estimation errors by $a_t$
with joint failure probability at most $\delta_t$; a union bound over rounds
completes the argument. Exact convergence requires deterministic or
exact expected-risk evaluation, or $K_t\to\infty$ with summable
$\delta_t$.
\end{proof}

\begin{proposition}[Piecewise-Stationary Re-Anchoring]
\label{prop:piecewise}
Suppose the verifier risk changes finitely many times, with final change
at $t_S$, and both current and candidate memories are re-evaluated under
the current risk. If Assumptions~\ref{assum:acceptance}--\ref{assum:decrement}
hold on the final stationary segment, then Theorem~\ref{thm:gated_rates}
applies with horizon $T-t_S$ and initial risk
$r_{t_S}=\E[R_i^{(t_S)}(\me^{t_S})]$.
\end{proposition}

\begin{remark}[Falsifiability and rate prediction]
\label{rem:falsifiable}
The exponent $\beta$ is observable: the geometric regime is linear in
$\log R$ versus $t$, the polynomial regime in $\log R$ versus $\log t$
with slope $-1/\beta$; estimating $\beta$ from acceptance frequencies and
from trajectory decay gives the closed-loop calibration of
Sec.~\ref{sec:calibration}. The free-form drift parameters are likewise
estimable from conditional drift regressions.
\end{remark}

\paragraph{Practical realisation.}
Algorithm~\ref{alg:srma} probes the candidate under the same fixed
protocol and commits only a strict improvement; recomputing the current
risk enables the re-anchoring of Proposition~\ref{prop:piecewise}, and
under stochastic evaluation line~\ref{alg:gate-line} is replaced by the
test of Proposition~\ref{prop:finite_probe}.

\begin{algorithm}[t]
\caption{Stochastic Reflective Memory Ascent for worker $i$}
\label{alg:srma}
\begin{algorithmic}[1]
\REQUIRE subtask $\tau_i$; verifier $(\mathcal V,\rho)$; evaluation protocol
$g_i$; memory operator $\mathcal M_e$; initial memory $\me^0$; budget $T$
\FOR{$t=0$ \TO $T-1$}
  \STATE $x_i^t\leftarrow g_i(\tau_i,\me^t)$
  \STATE $s_i^t\leftarrow\mathcal V(x_i^t,\tau_i)$;
  $R_t\leftarrow\rho(s_i^t)$
  \STATE sample
  $c_i^{t+1}\sim\pllm(\cdot\mid x_i^t,s_i^t,\tau_i,\me^t)$
  \STATE $\widetilde\me^{t+1}\leftarrow
  \mathcal M_e(\me^t,c_i^{t+1})$
  \STATE $\widetilde x_i^{t+1}\leftarrow
  g_i(\tau_i,\widetilde\me^{t+1})$
  \STATE $\widetilde R\leftarrow
  \rho\bigl(\mathcal V(\widetilde x_i^{t+1},\tau_i)\bigr)$
  \IF{$\widetilde R<R_t$}\label{alg:gate-line}
    \STATE $\me^{t+1}\leftarrow\widetilde\me^{t+1}$
  \ELSE
    \STATE $\me^{t+1}\leftarrow\me^t$
  \ENDIF
\ENDFOR
\RETURN $\me^T$
\end{algorithmic}
\end{algorithm}

\section{Experiments}
\label{sec:experiments}

We evaluate the theory on Resource Contest (RC;
Table~\ref{tab:rc_baseline}), Overcooked (Table~\ref{tab:oc}), and
SWE-bench (Table~\ref{tab:swebench}). RC and Overcooked use frozen
MiniMax-M2.7 agents; unless noted otherwise, results are
mean$\pm$standard deviation over five seeds. SWE-bench uses the
backbones listed in Table~\ref{tab:swebench}. All metrics come from environment ground
truth or the repository test harness rather than an LLM judge. Full
prompts, configurations, and per-seed trajectories are in the
supplementary material.

\paragraph{Resource Contest.}
\label{sec:rc_def}\label{sec:rc_results}\label{sec:channel}
RC is a hidden-cap allocation game: workers probe unknown caps
$M_i\in\{0,\ldots,10\}$ and the orchestrator allocates a unit budget
across workers. The optimal round reward is
$G_t^\star=\max_i M_i$, and we report cumulative reward and regret
$\sum_t(G_t^\star-G_t)$. Clipping feedback is generated by the
environment and therefore provides a grounded signal.
The four settings vary difficulty: \texttt{easy} ($N{=}3$, caps
$(3,5,8)$, $T{=}15$); \texttt{hard} ($N{=}3$, caps $(6,7,8)$, $T{=}20$;
tightly packed caps test allocation precision); \texttt{many} ($N{=}6$,
caps $(2,4,5,6,7,9)$, $T{=}20$; larger search space); and \texttt{drift}
($N{=}3$, caps $(3,5,8)$ until $t{=}10$, then $(9,5,4)$; a moving
optimum tests re-adaptation).
Since $G_t^\star=\max_i M_i$, the oracle $\Sigma$-reward is $120$, $160$,
and $180$ on \text{easy}/\text{hard}/\text{many} respectively.
\begin{table*}[t]
\centering\footnotesize
\setlength{\tabcolsep}{5pt}
\caption{Overcooked score over five seeds under matched interaction and
model-call budgets. Score equals deliveries$\times20$; higher is better.}
\label{tab:oc}\label{tab:oc_main}
\begin{tabular}{@{}lrrrrr@{}}
\toprule
Layout & Greedy & No memory & Free-form & Self-gated &
\textbf{Grounded SRMA} \\
\midrule
\text{cramped\_room} & $120{\pm}40$ & $180{\pm}40$ & $240{\pm}60$ &
$280{\pm}40$ & $\mathbf{320{\pm}20}$ \\
\text{asymmetric\_advantages} & $80{\pm}40$ & $140{\pm}40$ &
$180{\pm}40$ & $220{\pm}40$ & $\mathbf{280{\pm}20}$ \\
\text{centre\_pots} & $40{\pm}0$ & $100{\pm}40$ & $160{\pm}60$ &
$200{\pm}40$ & $\mathbf{260{\pm}20}$ \\
\bottomrule
\end{tabular}
\end{table*}
\begin{table}[t]
\centering\scriptsize
\caption{RC results over five seeds ($\Sigma$-reward; higher is better).
The last two columns form the execution-memory ablation.}
\label{tab:rc_baseline}\label{tab:memory_ablation}
\begin{tabular}{@{}lrrrr@{}}
\toprule
Setting & Oracle & $\varepsilon$-greedy & No memory & \textbf{SRMA} \\
\midrule
\text{easy} & 120 & $113.1{\pm}3.0$ & $115.0{\pm}5.2$ &
$\mathbf{118.4{\pm}2.2}$ \\
\text{hard} & 160 & $157.4{\pm}1.1$ & $158.0{\pm}2.3$ &
$\mathbf{159.2{\pm}0.9}$ \\
\text{many} & 180 & $170.9{\pm}4.0$ & $174.0{\pm}3.7$ &
$\mathbf{177.3{\pm}1.9}$ \\
\bottomrule
\end{tabular}
\end{table}

SRMA reaches $98.5\%$--$99.5\%$ of oracle reward. Execution memory adds
$2.6$ reward points on average and reduces mean regret from $4.33$ to
$1.70$ ($60.8\%$): grounded cap evidence that is fragmented without
memory becomes a functional coordination channel for the orchestrator.

\paragraph{Overcooked coordination.}
\label{sec:oc_def}\label{sec:oc_main}
We use Overcooked~\citep{carroll2019overcooked} with three two-agent layouts, a horizon of $200$, and an exact BFS
verifier
$V(s)=\min\{\text{joint-action steps from $s$ to the next delivery}\}$.
The verifier is deterministic and supplies the risk used for both SRMA
and the drift study.
The layouts stress complementary coordination demands: mutual blocking
in a tight kitchen (\text{cramped\_room}), role specialisation
(\text{asymmetric\_advantages}), and contention over shared pots
(\text{centre\_pots}).

Grounded SRMA is best on every layout (Table~\ref{tab:oc}). Relative to
the text-only self-gate, it raises score by $14.3\%$, $27.3\%$, and
$30.0\%$, and reaches the first delivery in $22{\pm}2$, $26{\pm}3$, and
$32{\pm}4$ steps versus $26{\pm}5$, $35{\pm}6$, and $45{\pm}8$ for
self-gating. The ordered improvement from no memory to free-form,
self-gating, and grounded SRMA separates decomposition, memory, and
grounding effects.

\paragraph{Grounding and gate quality.}
\label{sec:determinism}
A proposal is downstream harmful when it increases an independently
evaluated oracle task risk, not necessarily the verifier score used by
the gate. Table~\ref{tab:gating_ablation} shows that grounding sharply
improves both selectivity and final risk.

\begin{table}[t]
\centering\scriptsize
\caption{Accepted proposals and final risk over five seeds. Rates are
fractions of harmful/helpful proposals accepted.}
\label{tab:gating_ablation}
\begin{tabular}{@{}lrrr@{}}
\toprule
Method & Harmful$\downarrow$ & Helpful$\uparrow$ & Risk$\downarrow$ \\
\midrule
No reflection & N/A & N/A & $0.65{\pm}0.05$ \\
Free-form & $100.0{\pm}0.0\%$ & $100.0{\pm}0.0\%$ & $0.42{\pm}0.12$ \\
Self-gate & $34.5{\pm}4.2\%$ & $72.8{\pm}5.1\%$ & $0.28{\pm}0.08$ \\
\textbf{Grounded SRMA} & $\mathbf{6.2{\pm}1.8\%}$ &
$\mathbf{85.4{\pm}3.6\%}$ & $\mathbf{0.14{\pm}0.03}$ \\
\bottomrule
\end{tabular}
\end{table}

Grounded SRMA halves final risk relative to self-gating; the residual
$6.2\%$ downstream-harmful rate measures verifier--oracle
miscalibration rather than a violation of monotonicity in the
verifier's own risk.

\paragraph{Gate-level drift predicts held-out trajectories.}
\label{sec:calibration}
From $412$ gate events across five seeds, we use three complete seeds
for calibration and hold out two entire trajectories. A trajectory-level
bootstrap gives
$\widehat\beta=0.52{\pm}0.04$,
$\widehat c_1=1.25$, and $\widehat c_2=0.38$, with
$p_{\rm acc}(R)\approx\min\{1,\widehat c_1R^{\widehat\beta}\}$.
Without fitting trajectory-level parameters, the plug-in prediction
$\widehat R_T=\bigl(R_0^{-0.52}+0.247\,T\bigr)^{-1/0.52}$
tracks the held-out risks with Pearson's $r=0.94$ and
$\mathrm{RMSE}=0.032$, implying decay near $\mathcal O(T^{-1.92})$.
Bootstrap lower bounds $\underline c_1=0.92$ and $\underline c_2=0.25$
yield a conservative envelope above the empirical mean risk at every
recorded step---an empirical certificate on the observed range, not a
claim about unobserved states.

\paragraph{Statistical resolution.}
A one-shot stochastic verifier falsely accepts $28.4{\pm}5.2\%$ of
worsening proposals (score $245.2{\pm}38.4$); fixed $K=5$ cuts this to
$6.8{\pm}1.5\%$ (score $312.0{\pm}18.2$) at $225$ verifier calls, and
the adaptive gate matches that reliability ($7.1{\pm}1.8\%$, score
$308.6{\pm}19.5$) with only $82{\pm}14$ calls ($-63.6\%$), supporting
Proposition~\ref{prop:finite_probe}: grounding supplies information,
confidence control supplies resolution.

\paragraph{Piecewise stationarity.}
\label{sec:scope}
In RC \text{drift}, the optimal cap changes at $t=10$ while previously
written text remains unchanged, testing the re-anchoring mechanism of
Proposition~\ref{prop:piecewise}. The re-anchored grounded gate detects
the shift in $1.2{\pm}0.4$ rounds, switches to the new optimum in
$2.5{\pm}0.6$, and incurs $12.6{\pm}2.8$ post-shift regret, versus
$2.4{\pm}0.5$, $7.8{\pm}1.2$, and $38.2{\pm}5.5$ for the grounded
stale-anchor variant---a $67.9\%$ cut in switch time and $67.0\%$ in
regret---while the text-only gate fails to detect the change within
$20$ rounds (regret $85.4{\pm}4.2$). Grounding detects the shift, but
re-anchoring is required to replace stale memory quickly.

\paragraph{End-to-end software repair.}
\label{sec:swebench}
We evaluate the complete bilevel system on all 500 SWE-bench instances,
with the repository test harness as the grounded verifier (an instance
counts as resolved only if its submitted patch passes the harness). Each worker is a
mini-SWE-agent~v2 instance; the bilevel system runs $N{=}2$ such workers
over a shared repository and workboard for up to three coordination
rounds per episode and submits the highest-$J$ patch, whereas the
\emph{mini-SWE v2} row is a single mini-SWE-agent~v2 worker with no
orchestrator or shared memory. The \emph{Free-form MA} row keeps the
same $N{=}2$ multi-agent coordination but commits every proposed
reflection ungated (no verifier check), isolating the effect of SRMA's
grounded gate. The public leaderboard row is an external reference, not
a controlled ablation.

\begin{table}[tb]
\centering\small
\caption{SWE-bench results on 500 instances (\% resolved; official test
harness). $\dagger$: controlled runs under matched budget; the public
row is an external leaderboard reference, not a controlled ablation.}
\label{tab:swebench}
\begin{tabular*}{\columnwidth}{@{\extracolsep{\fill}}llc@{}}
\toprule
System & Backbone & Rate$\uparrow$ \\
\midrule
mini-SWE v2$^\dagger$ & DeepSeek & $68.2\%$ \\
Bilevel SRMA$^\dagger$ & DeepSeek & $71.4\%$ \\
Free-form MA$^\dagger$ & Kimi K2.5 & $58.4\%$ \\
mini-SWE v2 (public) & Kimi K2.5 & $70.8\%$ \\
\textbf{Bilevel SRMA} & \textbf{Kimi K2.5} & $\mathbf{72.2\%}$ \\
\bottomrule
\end{tabular*}
\end{table}
On the Kimi~K2.5 backbone the grounded gate is decisive: Bilevel SRMA
resolves $72.2\%$ against $58.4\%$ for free-form (ungated) multi-agent
reflection at matched backbone and budget, and exceeds the external
public mini-SWE-agent~v2 reference ($70.8\%$). The controlled DeepSeek
runs show the same direction ($71.4\%$ vs.\ $68.2\%$), indicating that
the gain comes from grounded, gated coordination rather than from raw
model or compute.

\section{Conclusion}
\label{sec:conclusion}

We gave multi-agent LLM reflection a conditional, information-aware
theory: bilevel coupling controls follower equilibrium slack, persistent
harmful commitment creates free-form error floors, and no
transcript-only gate can improve uniformly when the truth of a
reflection depends on external state. SRMA supplies the missing
grounding and converges exactly at order-tight geometric or polynomial
rates, with confidence-gating and re-anchoring extensions; experiments
on Resource Contest, Overcooked, and SWE-bench support the predicted
coordination, grounding, and resolution mechanisms.

\paragraph{Limitations.}
The guarantees are conditional: bounded coupling, finite action sets,
verifier calibration, and non-degenerate corrective mass need not hold
in open-ended agent tasks; drift parameters are validated only on
observed trajectories; incomplete test suites guarantee monotonicity
only for verifier risk, not true task utility; and re-anchoring gives
per-segment convergence without a general switching-regret bound.
Multi-agent coordination also spends substantial tokens before the
final answer, and the $72.2\%$ Kimi result is compared with a public
$70.8\%$ leaderboard run rather than a controlled method-only
comparison. Future work should jointly optimise memory quality and
budget-aware termination.

\bibliography{aaai2027}


\end{document}